\documentclass{article}
\usepackage{ijcai25}

\usepackage{times}
\usepackage{soul}
\usepackage{url}
\usepackage[hidelinks]{hyperref}
\usepackage[utf8]{inputenc}
\usepackage[small]{caption}
\usepackage{graphicx}
\usepackage{amsmath}
\usepackage{amsthm}
\usepackage{amssymb}
\usepackage{booktabs}
\usepackage{algorithm}
\usepackage{algorithmic}

\usepackage{tabularx} 
\usepackage{multirow} 
\usepackage{float}
\usepackage{array}
\usepackage[normalem]{ulem}
\useunder{\uline}{\ul}{}
\newtheorem{definition}{Definition}
\newtheorem{lemma}{Lemma}

\newtheorem{theorem}{Theorem}

\title{A Dynamic Stiefel Graph Neural Network for Efficient Spatio-Temporal Time Series Forecasting}

\author{
Jiankai Zheng
\and
Liang Xie\thanks{Corresponding Author: Liang Xie}
\affiliations
School of Mathematics and Statistics, Wuhan University of Technology, Wuhan 430070, China \\
\emails
312796@whut.edu.cn, whutxl@hotmail.com
}

\begin{document}

\maketitle

\begin{abstract}
Spatio-temporal time series (STTS) have been widely used in many applications. However, accurately forecasting STTS is challenging due to complex dynamic correlations in both time and space dimensions. Existing graph neural networks struggle to balance effectiveness and efficiency in modeling dynamic spatio-temporal relations. To address this problem, we propose the Dynamic Spatio-Temporal Stiefel Graph Neural Network (DST-SGNN) to efficiently process  STTS. For DST-SGNN, we first introduce the novel Stiefel Graph Spectral Convolution (SGSC) and Stiefel Graph Fourier Transform (SGFT). The SGFT matrix in SGSC is constrained to lie on the Stiefel manifold, and SGSC can be regarded as a filtered graph spectral convolution. We also propose the Linear Dynamic Graph Optimization on Stiefel Manifold (LDGOSM), which can efficiently learn the SGFT matrix from the dynamic graph and significantly reduce the computational complexity. Finally, we propose a multi-layer SGSC (MSGSC) that efficiently captures complex spatio-temporal correlations. Extensive experiments on seven spatio-temporal datasets show that DST-SGNN outperforms state-of-the-art methods while maintaining relatively low computational costs.
\end{abstract}

\section{Introduction}

In recent years, spatio-temporal time series (STTS) have become increasingly important in fields such as traffic flow forecasting, financial market analysis, and weather forecasting \cite{1,2}. These data encompass changes and relationships in both time and space dimensions. The spatial relationships, in a broader sense, not only represent geographical locations but also denote the relationships between other types of entities, such as companies or sensors. Accurately forecasting STTS is crucial for decision support and resource optimization in these areas.

In contrast to traditional time series forecasting, spatio-temporal forecasting faces the challenge of analyzing complex and dynamic spatio-temporal correlations among STTS data. The key problem is how to effectively model both the temporal dependencies and the spatial interactions between variables. Existing representative methods for multivariate time series forecasting, such as TimeMixer \cite{41} and CycleNet \cite{50}, have demonstrated excellent performance in capturing periodicity and trend changes in time series data. These methods utilize techniques such as decomposition to model complex temporal patterns effectively. However, they primarily focus on temporal dimensions and often overlook the spatial correlations in STTS, limiting their optimal performance in spatio-temporal forecasting. Graph Neural Networks (GNNs) \cite{12}, on the other hand, have gained popularity in spatio-temporal analysis due to their ability to model spatial relationships through graph structures. Recent advances in this area include the development of extended message-passing mechanisms for simultaneous spatial and temporal feature extraction \cite{51}, integrating spatio-temporal graph convolution networks for traffic prediction \cite{52}, and geometric masking techniques for interpreting GNNs in spatio-temporal contexts \cite{53}. These approaches fully utilize the advantages of GNNs in handling spatial interactions between nodes, thereby providing more accurate and comprehensive results for spatio-temporal forecasting.

However, although GNNs have great potential in spatio-temporal forecasting, their application faces two major challenges. Firstly, existing graph-based methods are computationally expensive, limiting their performance on large-scale datasets. Secondly, many graph methods struggle to adapt to the dynamic changes in spatio-temporal data, which further increases the computational costs \cite{10}. Although some graph Fourier methods, such as FourierGNN \cite{11}, have alleviated these issues to some extent, they still fall short in analyzing dynamic spatio-temporal data.

To address these challenges, this paper proposes the Dynamic Spatio-temporal Stiefel Graph Neural Network (DST-SGNN), which introduces the Stiefel Graph Spectral Convolution (SGSC) to effectively handle the dynamic changes and high-dimensional characteristics of spatio-temporal data. DST-SGNN first reduces the dimensionality and complexity of spatio-temporal data through patch-based sequence decomposition, breaking it into smaller, manageable patches to reduce redundancy and simplify the structure. Building on this, the model introduces Multi-layer Stiefel Graph Spectral Convolution (MSGSC) to effectively capture complex spatio-temporal correlations while maintaining high efficiency. To efficiently exploit dynamic graphs, Linear Dynamic Graph Optimization on Stiefel Manifold (LDGOSM) is proposed to efficiently learn the transformation matrix for SGSC.

The main advantages of DST-SGNN are as follows:

\begin{itemize}
\item{DST-SGNN can dynamically model the evolving spatial correlations over time, thereby providing more accurate forecasting results.}
\item{DST-SGNN achieves significant optimization in terms of time and space complexity compared to traditional graph neural networks, making it more efficient for processing large-scale spatio-temporal data.}
\item{Extensive experiments on seven spatio-temporal datasets demonstrate DST-SGNN's superiority in accuracy and efficiency in comparison with state-of-the-art spatio-temporal and time series forecasting methods.}
\end{itemize}

\section{Related Work}
\subsection{Spatio-Temporal Time Series Forecasting}

Multivariate time series forecasting is a key application in time series analysis, aiming to predict future values by modeling multiple correlated variables. Patch-based methods (such as PatchTST \cite{54} and MTST \cite{55} ) have made significant progress by dividing time series into local segments and using the Transformer architecture to capture both local and global temporal dependencies. Meanwhile, methods based on time series decomposition, such as TimesNet \cite{14} and TimeMixer \cite{41}, have further improved the performance of long-sequence forecasting by optimizing the decomposition process, and they better capture the intrinsic temporal structure of time series.

In the field of STTS, researchers have developed several innovative methods. STPGNN \cite{56} identifies pivotal nodes and captures their complex spatio-temporal dependencies to enhance traffic prediction. MiTSformer \cite{44} addresses mixed time series by recovering latent continuous variables and leveraging multi-scale context for balanced modeling. ModWaveMLP \cite{57} offers an efficient and simple solution for traffic forecasting using MLPs combined with mode decomposition and wavelet denoising. Lastly, GPT-ST \cite{2} introduces a spatio-temporal pre-training framework that enhances downstream models' learning capabilities through masked autoencoders and adaptive masking strategies. These methods, each with unique strengths, provide diverse solutions for spatio-temporal time series forecasting.

\subsection{Graph Neural Networks for Spatio-Temporal Analysis}

GNNs are widely applied in spatio-temporal prediction. Their powerful modeling capabilities can effectively capture the complex relationships within the data. In the context of spatio-temporal prediction, the graph structure needs to be dynamically updated according to the data to adapt to spatio-temporal changes \cite{58,59}. Meanwhile, as demonstrated by LLGformer \cite{63} and MPFGSN \cite{64}, full spatio-temporal graphs can significantly improve the prediction performance and help uncover potential patterns. However, the dynamic update and the construction of full spatio-temporal graphs will incur significant computational overhead. Especially when dealing with large-scale dynamic data, this limits the widespread application of GNNs.

Graph spectral convolution methods perform convolution in the spectral domain based on the graph Fourier transform, which can effectively capture the complex relationships in STTS \cite{34}. Nevertheless, traditional methods require eigenvalue decomposition of the adjacency matrix, resulting in high computational costs. Methods such as GCN \cite{12,61} and ChebNet \cite{32,62} replace eigenvalue decomposition with a spatial-like convolution that aggregates information via the adjacency matrix. However, the construction and dynamic update of the adjacency matrix remain computationally expensive. The FourierGNN \cite{11} uses the standard Fourier transform to improve computational efficiency, but it is only applicable to fixed graph structures and is not suitable for dynamic spatio-temporal relationships.

\begin{figure*}[htbp]  
    \centering 
    \includegraphics[width=1\linewidth]{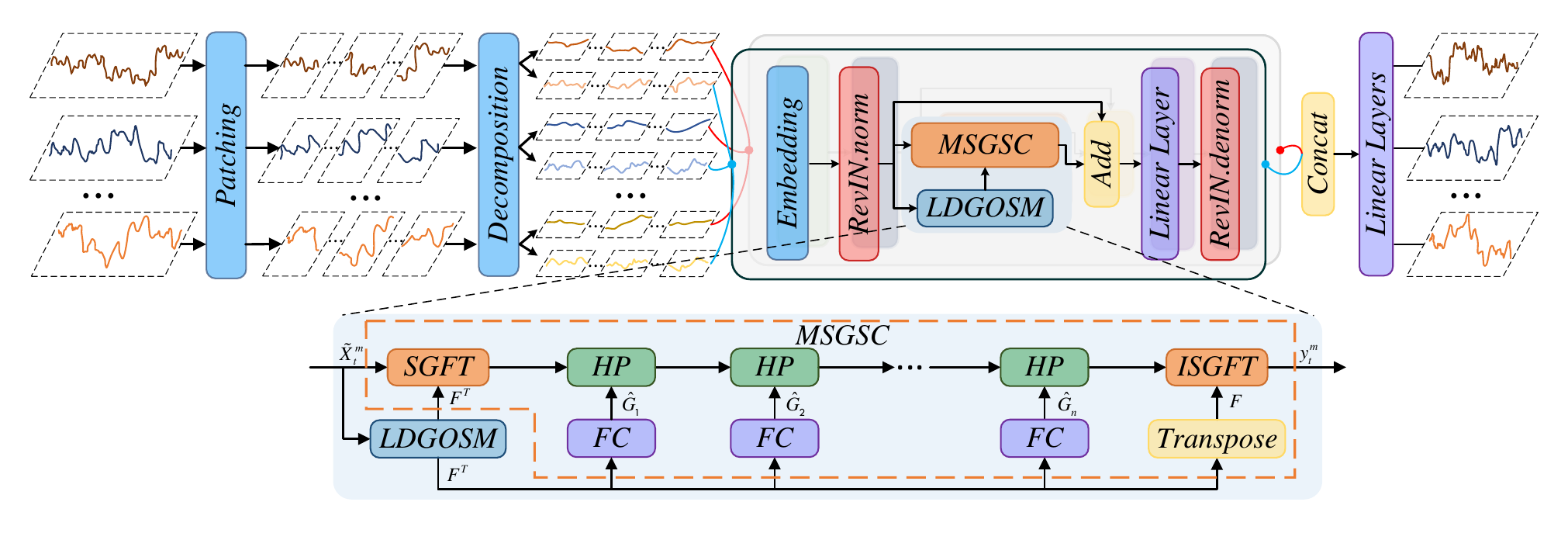} 
    \caption{Framework diagram of the DST-SGNN model.} 
    \label{fig:dst-sgnn}
\end{figure*}

\section{Problem Definition}

The problem of STTS forecasting involves forecasting future values of multiple variables based on their past observations across both time and space dimensions. STTS is similar to multivariate time series, and the main difference between them is that there exist spatial correlations among variables of STTS. Specifically, given a time series dataset with $N$ variables observed over $T$ consecutive time points, the goal is to predict the values of these variables at future time points $t+1$ to $t+\tau $, where $\tau $ is the prediction length. The prediction function ${{\mathbf{F}}_{\theta }}$ takes historical data ${{X}_{t}}$ as input and outputs the predicted values ${{Y}_{t}}$ for the next $\tau $ time points, capturing complex patterns and dependencies in the STTS data.

Formally, the prediction problem can be defined as finding a function ${{\mathbf{F}}_{\theta }}$ such that ${{Y}_{t}}={{\mathbf{F}}_{\theta }}({{X}_{t}})$, where ${{Y}_{t}}$ represents the predicted multivariate values for future time points $t+1$ to $t+\tau $. 

\section{Methodology}

This section details the DST-SGNN model methodology. Section 4.1 presents the overall framework, while Sections 4.2–4.5 cover module implementation, with Section 4.4 focusing on the core theory of the Stiefel graph neural network. Section 4.6 describes training and inference procedures.

\subsection{Overall Framework}

Figure \ref{fig:dst-sgnn} illustrates the overall framework of the DST-SGNN model, which mainly includes four modules: Patching, Decomposition, LDGOSM, and MSGSC. First, the Patching module decomposes the original high-dimensional spatio-temporal data into multiple smaller sub-sequences (patches) to reduce data complexity. Next, the Decomposition module further decomposes each patch to extract seasonal and trend components, which are used as nodes to construct a hyperpatch graph.  After constructing the hyperpatch graph, an initial transformation matrix is formed from the graph's data. Subsequently, the LDGOSM module dynamically optimizes the transformation matrix on the Stiefel manifold. Finally, the MSGSC module performs multi-step graph spectral convolutions,  where $HP$ denotes the element-wise Hadamard product operation in the spectral convoluation process. The seasonal and trend components are concatenated through the Concat module, and the resulting features are fed into the Linear Layers module to generate the final predictions. In the following, we will introduce each module in detail.

\subsection{Patching and Decomposition}

\subsubsection{Patching}
Given a multivariate time series window, the input  is  ${{X}_{t}}\in {{\mathbb{R}}^{T\times N}}$ of $N$ variables at timestamp $t$. Let $p$ denote the patch size, and $s$ denote the stride (the length of the non-overlapping region between two consecutive patches). Then, ${{X}_{t}}$ is fed into the $Patching$ module $\mathcal{T}$$:{{\mathbb{R}}^{T\times N}}\to {{\mathbb{R}}^{J\times p\times N}}$. This module transforms ${{X}_{t}}$ into $J=\left\lfloor (T-p)/s \right\rfloor +2$ patches. The specific calculation for the $Patching$ module ${{\hat{X}}_{t}}=\mathcal{T}({{X}_{t}})$ is as follows:
\begin{equation}
\hat{X}_{t,j}={{X}_{t}}_{[s(j-1)+1:s(j-1)+p]},(j=1,2,...,J)
\end{equation}
\noindent
where ${{\hat{X}}_{t}}\in {{\mathbb{R}}^{J\times p\times N}}$ represents the output of the $Patching$ module, and ${{\hat{X}}_{t,j}}$ is its $j$-th column. To ensure that $J$ is an integer, we pad the last $s$ values of the sequence to the end of the original sequence.

\subsubsection{Intra-patch Sequence Decomposition for Constructing the Hyperpatch Graph}

First, the time series ${{\hat{X}}_{t}}$ is decomposed into seasonal component $\hat{X}_{t}^{(1)}$ and trend component $\hat{X}_{t}^{(2)}$ using the series decomposition block of Autoformer \cite{3}.
\begin{equation}
\hat{X}_{t}^{(1)},\hat{X}_{t}^{(2)}=\text{SeriesDecomp}({{\hat{X}}_{t}})
\end{equation}

In the hyperpatch graph structure, each patch in the seasonal component $\hat{X}_{t}^{(1)}$ and trend component $\hat{X}_{t}^{(2)}$ is defined as a node in the graph. For the sliding window of the time series, each element within the window is considered to have a connection with every other element, thus it is represented in the form of a spatio-temporal fully connected graph. Based on $\hat{X}_{t}^{m}\in {{\mathbb{R}}^{J\times N\times p}},m=(1),(2)$, a hyperpatch graph structure $G_{t}^{m}=(X_{t}^{G,m},A_{t}^{G,m})$ is further constructed. This structure is initialized as a fully connected graph containing $(J\times N)\times p$ nodes, where $X_{t}^{G,m}\in {{\mathbb{R}}^{(J\times N)\times p}}$ represents the node features, and $A_{t}^{G,m}\in {{\mathbb{R}}^{((J\times N)\times p)\times ((J\times N)\times p)}}$ is the adjacency matrix, used to describe the relationships between nodes.

\subsection{Embedding and Normalization}

In the DST-SGNN model, the Embedding module maps the data $X_{t}^{G,m}\in {{\mathbb{R}}^{(J\times N)\times p}}$ after Decomposition into the hidden space $\tilde{X}_{t}^{m}\in {{\mathbb{R}}^{(J\times N)\times K}}$, where $K$ is the dimension of the hidden layer. This process reduces the data dimensionality and enhances the feature representation capability. The Normalization module normalizes each time series instance to have zero mean and unit variance, thereby reducing the distribution differences between training and testing data \cite{49}. This design improves the model's generalization ability and adaptability to different data distributions.

\subsection{Stiefel Graph Spectral Convolution}

\subsubsection{Basic Concepts}

\begin{definition}
The Stiefel Graph Spectral Convolution of a signal \( x \in \mathbb{R}^{n \times 1} \) with a kernel \( g \in \mathbb{R}^{n \times 1} \) is defined as:
\begin{equation}
x *_{s} g = F(F^T x \odot F^T g)
\end{equation}
where \( F \) satisfies the following optimization problem:
\begin{equation}
\begin{aligned}
& \text{Min} \quad \text{Tr}(F^T L F) \\ 
& \text{s.t.} \quad F \in St(n, d) \\ 
\end{aligned}
\label{fun1}
\end{equation}
\label{def1}
\end{definition}

In Definition \ref{def1}, $\odot $ denotes the Hadamard product, \( St(n, d) \triangleq \{ F \in \mathbb{R}^{n \times d}: F^T F = I_d \} \) is the Stiefel manifold, \( L = I_n - D^{-1/2} A D^{-1/2} \) is the Laplacian matrix, \( A \in \mathbb{R}^{n \times n} \) is the graph adjacency matrix, and \( D \) is the diagonal matrix, and its diagonal element is the sum of all elements in each row of \( A \), which represents the degree of the corresponding node in an undirected graph (i.e., the number of edges connected to the node).

\begin{theorem}
The matrix $F$ can be obtained by solving the eigenvalue decomposition of ${{D}^{-1/2}}A{{D}^{-1/2}}$ and selecting the eigenvectors corresponding to the d largest eigenvalues.
\label{the1}
\end{theorem}

The proof of Theorem \ref{the1} is provided in Appendix \ref{appA.1}.

The ${{F}^{T}}x$ in Definition \ref{def1} can be considered as a pseudo-graph Fourier transform, which we refer to as the Stiefel Graph Fourier Transform (SGFT), defined as:

\begin{equation}
S(x)={{F}^{T}}x
\end{equation}
\noindent
We can also define the Inverse Stiefel Graph Fourier Transform (ISGFT) as:

\begin{equation}
{{S}^{-1}}(x)=Fx
\end{equation}

The SGSC of $x$ and $g$ can be obtained by first applying the SGFT to $x$ and $g$ separately, then performing element-wise multiplication, and finally applying the ISGFT. Since $F$ is not orthogonal, it does not satisfy the definition of the standard graph Fourier transform. However, we can prove that SGSC is still a specific type of standard graph spectral convolution \cite{32}.

\begin{theorem}

The Stiefel Graph Spectral Convolution can be viewed as the following specific filtered graph spectral convolution:

\begin{equation}
x{{*}_{s}}g=P{{g}_{\theta }}(\Lambda ){{P}^{T}}x
\end{equation}
\noindent
where $\Lambda =diag({{\lambda }_{1}},{{\lambda }_{2}},\cdots ,{{\lambda }_{n}})$, and ${{\lambda }_{1}},{{\lambda }_{2}},\cdots ,{{\lambda }_{n}}$ are the eigenvalues of $A$ sorted in descending order. $P=({{p}_{1}},{{p}_{2}},\cdots ,{{p}_{n}})$, where ${{p}_{i}}\in {{R}^{n\times 1}}$ is the eigenvector corresponding to ${{\lambda }_{i}}$. ${{g}_{\theta }}(\Lambda )=diag(\theta )$, and $\theta =({{\theta }_{1}},{{\theta }_{2}},\cdots ,{{\theta }_{n}})$ satisfies:
\begin{equation}
\theta_i = 
\begin{cases} 
p_i^T g & \text{if } \lambda_i \ge \lambda_d \\
0 & \text{if } \lambda_i < \lambda_d 
\end{cases}
\end{equation}
\label{the2}
\end{theorem}

The proof of Theorem \ref{the2} is provided in Appendix \ref{appA.2}.

According to Theorem \ref{the2}, it can be seen that the Stiefel Graph Spectral Convolution satisfies all the properties of graph spectral convolution. Moreover, setting the coefficients ${{\theta }_{i}}$ corresponding to smaller eigenvalues to zero reduces the time and space complexity of the graph spectral convolution. And it also serves as a filtering and noise-reduction mechanism.

The SGSC we define can be easily extended to the case of multi-dimensional features. For $X\in {{R}^{n\times k}}$, if we choose the convolution kernel $G\in {{R}^{n\times k}}$, then the SGSC is given by:
\begin{equation}
X{{*}_{s}}G=F({{F}^{T}}X\odot {{F}^{T}}G)
\end{equation}

\subsubsection{Linear Dynamic Graph Optimization on Stiefel Manifold}

\begin{table*}[htbp]
\centering
\begin{tabular}{c|cccccrr}
\hline
Datasets  & PEMS03 & PEMS07 & Electricity & CSI300 & exchange\_rate & Solar & METR-LA \\ \hline
Features  & 358    & 883    & 321         & 100    & 8              & 593   & 207     \\
Timesteps & 26207  & 28223  & 26304       & 2791   & 7588           & 8760  & 34272   \\ \hline
\end{tabular}
\caption{Dataset statistics.} 
\label{tab1}
\end{table*}

In the SGSC, the introduction of the Stiefel manifold simplifies the convolution computation, but the eigenvalue decomposition of $L$ still results in a time complexity of $O({{N}^{3}})$, where $N$ represents the total number of nodes across all variables and all patches in our paper. Moreover, for spatio-temporal data, the spatial relationships between variables change over time, requiring graph adjacent matrix $A$ to be recalculated frequently. Each computation of the graph requires solving for eigenvalues, leading to a more complex computation. To simplify the dynamic computation of $A$, inspired by \cite{48}, we introduce the following calculation of the dynamic graph adjacency matrix:

\begin{equation}
A={{I}_{N}}+E{{E}^{T}},\,\,\,E=ReLU(X)
\label{fun2}
\end{equation}

Furthermore, to reduce the complexity of the SGSC, we improve the computation of $F$ in Theorem \ref{the1}. We propose an efficient Linear Dynamic Graph Optimization on Stiefel Manifold (LDGOSM). Inspired by \cite{36}, we introduce a linear transformation:
$F=XW$,
where $W\in \mathbb{R}{{}^{d\times d}}$. By substituting equation (\ref{fun2}) into the equivalent form of (\ref{fun1}) (as proven in Theorem \ref{the1} in Appendix \ref{appA.1}), the optimization problem can be transformed into the following objective function:
\begin{equation}
\begin{aligned}
& \max \quad \text{Tr}(W^T (X^T X + X^T E E^T X) W) \\
& \text{s.t.} \quad \quad \quad \quad \quad W^T X^T X W = I \\
\end{aligned}
\label{fun4}
\end{equation}
The optimal solution $W$ for the objective function (\ref{fun4}) can be obtained through Algorithm \ref{alg1} (see Appendix \ref{appA.3} for detailed derivation).

\begin{algorithm}[tb]
    \caption{Implementation of LDGOSM}
    \label{alg1}
    \textbf{Input}: $X\in {{R}^{n\times d}}$, $E\in {{R}^{n\times d}}$\\
    \textbf{Output}: $W\in {{R}^{d\times d}}$
    \begin{algorithmic}[1] 
        \STATE Perform eigenvalue decomposition on $B={{X}^{T}}X$ to obtain eigenvector matrix $D$ and eigenvalue diagonal matrix $\Lambda $;
        \STATE Let $M=D{{\Lambda }^{-1/2}}{{D}^{T}}$;
        \STATE Compute $C={{E}^{T}}X$;
        \STATE Compute $H=B+{{C}^{T}}C$;
        \STATE Perform eigenvalue decomposition on ${{M}^{T}}HM$ to obtain eigenvector matrix $U$;
        \STATE \textbf{return} $W=MU$;
    \end{algorithmic}
\end{algorithm}

In Algorithm \ref{alg1}, the complexity matrix $B$ and the eigenvalue decomposition are $O(n{{d}^{2}})$ and $O({{d}^{3}})$, respectively. Computing matrices $M$, $C$, and $H$ are $O({{d}^{3}})$, $O(n{{d}^{2}})$, and $O({{d}^{2}})$. Computing ${{M}^{T}}HM$ and the subsequent eigenvalue decomposition are both $O({{d}^{3}})$. Finally, computing matrix $W$ is $O({{d}^{3}})$. The total complexity is $O(n{{d}^{2}}+k{{d}^{3}})$, where $k$ is a small constant. If $d\ll n$, the complexity is dominated by $O(n{{d}^{2}})$, making it approximately linear with nodes number $n$.

\subsubsection{Multi-layer SGSC}

To enhance the effectiveness of dynamic spatio-temporal graph learning, we employ the following Multi-layer SGSC (MSGSC):

\begin{equation}
MSGSC(X,G)=\sum\limits_{i=1}^{m}{X{{*}_{s}}{{G}_{1}}{{*}_{s}}}{{G}_{2}}{{*}_{s}}\cdots {{*}_{s}}{{G}_{i}}
\end{equation}

\noindent
where ${{G}_{i}}$ represents the convolutional kernel of each layer in the MSGSC.

\begin{theorem}
The MSGSC has the following equivalent computational form:
\begin{equation}
MSGSC(X,G)={{S}^{-1}}(\sum\limits_{i=1}^{m}{S(X)}\odot \prod\limits_{j=1}^{i}{S({{G}_{j}})})
\end{equation}
\label{the3}
\end{theorem}
The proof of Theorem \ref{the3} is provided in Appendix \ref{appA.4}.

Through Theorem \ref{the3}, when computing the MSGSC, multiple SGFT and ISGFT operations are avoided, thereby improving computational efficiency while retaining the effects of the original multi-step convolution.

In our implementation, $S({{G}_{j}})={{F}^{T}}{{G}_{j}}$ can be learned via a linear fully connected layer (FC), i.e., ${{\hat{G}}_{j}}=FC({{F}^{T}})$. Therefore, the computation of MSGSC can be expressed as:

\begin{equation}
\begin{aligned}
y_{t}^{m}=MSGSC(\tilde{X}_{t}^{m},{F}^{T})={{S}^{-1}}(\sum\limits_{i=1}^{m}{S(\tilde{X}_{t}^{m})}\odot \prod\limits_{j=1}^{i}{{{{\hat{G}}}_{j}}})
\end{aligned}
\end{equation}
\noindent
where, $F$ is optimized using the Algorithm \ref{alg1}. 

\subsection{Concat Module and Linear Layers Module}

In DST-SGNN, the seasonal and trend components $y_{t}^{(1)}$ and $y_{t}^{(2)}$ output from the MSGSC module are processed through inverse RevIN, resulting in $\hat{Y}_{t}^{(1)}$ and $\hat{Y}_{t}^{(2)}$. These features are then concatenated via the Concat module to form ${{\hat{Y}}_{t}}$, which is subsequently fed into the Linear Layers module to generate the final prediction ${{Y}_{t}}$.

\subsection{The Inference and Training of DST-SGNN}
In DST-SGNN, spatio-temporal data is decomposed into patches and split into seasonal and trend components. The LDGOSM module optimizes the transformation matrix $F$ on the Stiefel manifold. The MSGSC module captures complex patterns via graph spectral convolutions. Predictions are generated by concatenating features and passing them through linear layers. The training process involves only standard eigenvalue decomposition, which can be efficiently handled using PyTorch's linalg module, and uses MAE loss for optimization.

\section{Experiments}

\begin{table*}[htbp]
\centering
\footnotesize 
\setlength{\tabcolsep}{0.7mm}{
\resizebox{\textwidth}{!}{%
\begin{tabular}{cc|cc|cc|cc|cc|cc|cc|cc|cc|cc}
\hline
\multicolumn{2}{c|}{Model}                                 & \multicolumn{2}{c|}{DST-SGNN}      & \multicolumn{2}{c|}{TimeMixer} & \multicolumn{2}{c|}{TimeXer}      & \multicolumn{2}{c|}{MiTSformer} & \multicolumn{2}{c|}{iTransformer} & \multicolumn{2}{c|}{FourierGNN} & \multicolumn{2}{c|}{FreTS}        & \multicolumn{2}{c|}{Fredformer} & \multicolumn{2}{c}{FITS} \\ \hline
\multicolumn{2}{c|}{Metric}                                & MSE             & MAE             & MSE             & MAE          & MSE             & MAE             & MSE            & MAE            & MSE             & MAE             & MSE            & MAE            & MSE             & MAE             & MSE           & MAE             & MSE            & MAE     \\ \hline
\multicolumn{1}{c|}{\multirow{4}{*}{PEMS03}}         & 96  & {\ul 0.1936}    & \textbf{0.2635} & 0.3607          & 0.4203       & \textbf{0.1820} & 0.2962          & 1.3222         & 0.9412         & 1.7197          & 1.0753          & 0.2978         & 0.4049         & 0.2459          & 0.3503          & 1.1039        & 0.8235          & 0.4239         & 0.5127  \\
\multicolumn{1}{c|}{}                                & 192 & \textbf{0.2122} & \textbf{0.2760} & 0.4477          & 0.4660       & {\ul 0.2431}    & 0.3452          & 1.7115         & 1.0879         & 1.9911          & 1.2053          & 0.3284         & 0.4251         & 0.2808          & 0.3751          & 1.1645        & 0.8395          & 0.3541         & 0.4484  \\
\multicolumn{1}{c|}{}                                & 336 & \textbf{0.2318} & \textbf{0.2910} & 0.3431          & 0.3970       & 0.2889          & 0.3758          & 1.3013         & 0.8972         & 1.3191          & 0.9035          & 0.3134         & 0.4104         & {\ul 0.2618}    & 0.3541          & 0.8229        & 0.6543          & 0.3106         & 0.4087  \\
\multicolumn{1}{c|}{}                                & 720 & 0.3040          & \textbf{0.3506} & 0.4247          & 0.4594       & {\ul 0.3002}    & 0.3800          & 1.5214         & 0.9890         & 1.6422          & 1.0467          & 0.3485         & 0.4354         & \textbf{0.2997} & 0.3842          & 0.8356        & 0.6663          & 0.4988         & 0.5591  \\ \hline
\multicolumn{1}{c|}{\multirow{4}{*}{PEMS07}}         & 60  & 0.4696          & 0.4417          & 0.2905          & 0.3802       & \textbf{0.2101} & \textbf{0.3277} & 0.8547         & 0.7161         & 1.6041          & 1.0468          & 0.2644         & 0.3668         & {\ul 0.2403}    & {\ul 0.3357}    & 0.8227        & 0.6662          & 0.3207         & 0.4345  \\
\multicolumn{1}{c|}{}                                & 96  & 0.4645          & 0.4444          & 0.3826          & 0.4261       & 0.4625          & 0.5196          & 1.2500         & 0.8942         & 1.8446          & 1.1342          & {\ul 0.3178}   & {\ul 0.4109}   & \textbf{0.2860} & \textbf{0.3699} & 1.5280        & 0.9913          & 0.4041         & 0.5043  \\
\multicolumn{1}{c|}{}                                & 192 & \textbf{0.3436} & \textbf{0.3762} & 0.4582          & 0.4739       & 0.5380          & 0.5585          & 1.4574         & 0.9694         & 1.7269          & 1.0716          & 0.3822         & 0.4379         & {\ul 0.3479}    & {\ul 0.4093}    & 1.5146        & 0.9855          & 0.3689         & 0.4674  \\
\multicolumn{1}{c|}{}                                & 336 & 0.3346          & \textbf{0.3675} & 0.3624          & 0.4091       & 0.4565          & 0.4973          & 1.1288         & 0.8114         & 1.4432          & 0.9504          & 0.3462         & 0.4148         & \textbf{0.3129} & {\ul 0.3844}    & 0.8725        & 0.6821          & {\ul 0.3327}   & 0.4399  \\ \hline
\multicolumn{1}{c|}{\multirow{4}{*}{Electricity}}    & 96  & 0.1552          & 0.2448          & 0.1586          & 0.2494       & \textbf{0.1441} & {\ul 0.2438}    & 0.1558         & 0.2483         & {\ul 0.1460}    & \textbf{0.2381} & 0.2931         & 0.3838         & 0.1792          & 0.2666          & 0.2344        & 0.3201          & 0.5628         & 0.6111  \\
\multicolumn{1}{c|}{}                                & 192 & 0.1671          & {\ul 0.2555}    & 0.1706          & 0.2606       & \textbf{0.1618} & 0.2602          & 0.1741         & 0.2634         & {\ul 0.1636}    & \textbf{0.2553} & 0.3030         & 0.3927         & 0.1837          & 0.2731          & 0.2382        & 0.3259          & 0.5974         & 0.6299  \\
\multicolumn{1}{c|}{}                                & 336 & {\ul 0.1806}    & \textbf{0.2707} & 0.1882          & 0.2788       & 0.1829          & 0.2820          & 0.2021         & 0.2893         & \textbf{0.1803} & {\ul 0.2720}    & 0.3173         & 0.4064         & 0.1996          & 0.2905          & 0.2556        & 0.3418          & 0.6450         & 0.6547  \\
\multicolumn{1}{c|}{}                                & 720 & {\ul 0.2133}    & {\ul 0.2999}    & 0.2296          & 0.3159       & 0.2162          & 0.3129          & 0.2332         & 0.3138         & \textbf{0.2082} & \textbf{0.2975} & 0.3456         & 0.4256         & 0.2353          & 0.3244          & 0.4009        & 0.4527          & 0.6437         & 0.6490  \\ \hline
\multicolumn{1}{c|}{\multirow{4}{*}{CSI300}}         & 36  & \textbf{0.1435} & \textbf{0.2182} & 0.2121          & 0.2625       & {\ul 0.1845}    & {\ul 0.2580}    & 0.2574         & 0.2970         & 0.2285          & 0.2862          & 0.4066         & 0.3398         & 0.2148          & 0.2757          & 0.3369        & 0.3655          & 0.2511         & 0.3733  \\
\multicolumn{1}{c|}{}                                & 48  & \textbf{0.1876} & \textbf{0.2487} & 0.3067          & 0.3098       & {\ul 0.2269}    & {\ul 0.2792}    & 0.3401         & 0.3411         & 0.2803          & 0.3159          & 0.6206         & 0.4078         & 0.3078          & 0.3197          & 0.3722        & 0.3863          & 0.5647         & 0.6130  \\
\multicolumn{1}{c|}{}                                & 60  & \textbf{0.2241} & \textbf{0.2717} & 0.3406          & 0.3350       & 0.2648          & {\ul 0.3024}    & 0.5111         & 0.4111         & 0.3762          & 0.3729          & 0.9337         & 0.4849         & 0.6574          & 0.4090          & 0.4013        & 0.4030          & {\ul 0.2586}   & 0.3798  \\
\multicolumn{1}{c|}{}                                & 96  & \textbf{0.2987} & \textbf{0.3259} & 0.5120          & 0.4107       & 0.5740          & 0.4651          & 0.6084         & 0.4585         & {\ul 0.4247}    & {\ul 0.4017}    & 1.2942         & 0.5696         & 1.1953          & 0.5384          & 0.4838        & 0.4507          & 0.6159         & 0.6428  \\ \hline
\multicolumn{1}{c|}{\multirow{4}{*}{exchange\_rate}} & 48  & \textbf{0.0439} & {\ul 0.1464}    & 0.0572          & 0.1631       & 0.0919          & 0.2123          & {\ul 0.0474}   & 0.1486         & 0.0917          & 0.2213          & 0.0726         & 0.1957         & 0.0532          & 0.1634          & 0.0794        & \textbf{0.0794} & 0.1456         & 0.2887  \\
\multicolumn{1}{c|}{}                                & 96  & \textbf{0.0791} & \textbf{0.2020} & 0.1184          & 0.2394       & 0.1155          & 0.2513          & 0.1148         & {\ul 0.2291}   & {\ul 0.1042}    & 0.2338          & 0.1516         & 0.2905         & 0.1218          & 0.2515          & 0.1241        & 0.2519          & 0.2278         & 0.3653  \\
\multicolumn{1}{c|}{}                                & 192 & \textbf{0.1508} & {\ul 0.2860}    & 0.2405          & 0.3454       & 0.2019          & \textbf{0.2019} & {\ul 0.1857}   & 0.2957         & 0.2175          & 0.3448          & 0.2953         & 0.4134         & 0.1915          & 0.3243          & 0.2215        & 0.3399          & 0.4276         & 0.5084  \\
\multicolumn{1}{c|}{}                                & 336 & \textbf{0.2611} & \textbf{0.3812} & 0.4107          & 0.4612       & 0.3794          & 0.4486          & 0.3921         & {\ul 0.4390}   & 0.3879          & 0.4562          & 0.4328         & 0.4985         & 0.4050          & 0.4825          & {\ul 0.3688}  & 0.4437          & 0.6705         & 0.6398  \\ \hline
\multicolumn{1}{c|}{\multirow{4}{*}{Solar}}          & 96  & 0.1802          & \textbf{0.2007} & \textbf{0.1606} & 0.2216       & 0.1677          & 0.2472          & 0.1807         & {\ul 0.2532}   & {\ul 0.1631}    & 0.2191          & 0.3143         & 0.4134         & 0.1955          & 0.2482          & 0.3434        & 0.4379          & 0.7187         & 0.7212  \\
\multicolumn{1}{c|}{}                                & 192 & 0.1822          & \textbf{0.2003} & \textbf{0.1593} & 0.2289       & 0.1937          & 0.2657          & 0.1875         & {\ul 0.2632}   & {\ul 0.1663}    & 0.2226          & 0.3103         & 0.4207         & 0.2020          & 0.2478          & 0.3343        & 0.4409          & 0.7260         & 0.7284  \\
\multicolumn{1}{c|}{}                                & 336 & 0.1849          & \textbf{0.1989} & \textbf{0.1624} & 0.2295       & 0.1745          & 0.2596          & 0.1874         & {\ul 0.2536}   & {\ul 0.1706}    & 0.2272          & 0.3114         & 0.4217         & 0.2038          & 0.2538          & 0.3907        & 0.4743          & 0.6669         & 0.6959  \\
\multicolumn{1}{c|}{}                                & 720 & 0.2021          & \textbf{0.2102} & \textbf{0.1709} & 0.2361       & {\ul 0.1812}    & 0.2722          & 0.1989         & {\ul 0.2647}   & 0.1830          & 0.2323          & 0.4217         & 0.4295         & 0.2323          & 0.2674          & 0.3496        & 0.4652          & 0.7663         & 0.7561  \\ \hline
\multicolumn{1}{c|}{\multirow{4}{*}{METR-LA}}        & 96  & 1.2756          & \textbf{0.6243} & {\ul 1.1316}    & {\ul 0.6483} & 1.4146          & 0.7436          & 1.3261         & 0.7250         & 1.5068          & 0.7609          & 1.1737         & 0.7085         & \textbf{1.0759} & 0.6834          & 1.2441        & 0.6516          & 1.2678         & 0.7873  \\
\multicolumn{1}{c|}{}                                & 192 & 1.3912          & \textbf{0.6619} & 1.3192          & {\ul 0.7097} & 1.5473          & 0.7941          & 1.4698         & 0.7821         & 1.6180          & 0.7923          & {\ul 1.2967}   & 0.7678         & \textbf{1.2428} & 0.7310          & 1.5024        & 0.7636          & 1.3339         & 0.8080  \\
\multicolumn{1}{c|}{}                                & 336 & 1.4780          & \textbf{0.6841} & 1.3539          & {\ul 0.7247} & 1.5242          & 0.7706          & 1.5705         & 0.7883         & 1.6485          & 0.7838          & {\ul 1.3287}   & 0.7842         & \textbf{1.2913} & 0.7513          & 1.5445        & 0.7613          & 1.4403         & 0.8430  \\
\multicolumn{1}{c|}{}                                & 720 & 1.7026          & \textbf{0.7302} & 1.5596          & 0.8309       & 1.7827          & 0.8716          & 1.7972         & 0.8438         & 1.9614          & 0.8813          & {\ul 1.4944}   & 0.8244         & \textbf{1.4582} & {\ul 0.7957}    & 1.8800        & 0.8500          & 1.6439         & 0.9186  \\ \hline
\multicolumn{2}{c|}{avg rank}                              & \textbf{2.7857} & \textbf{1.5000} & 3.8929          & {\ul 3.8214} & {\ul 3.7857}    & 4.2143          & 5.9286         & 5.6786         & 5.5714          & 5.3929          & 5.8571         & 6.1071         & 3.8929          & 4.2143          & 6.7500        & 6.3214          & 6.5357         & 7.7500  \\ \hline
\end{tabular}}}
\caption{Summary of results for all methods on nine datasets.The best
results are in bold and the second best are underlined.} 
\label{tab2}
\end{table*}

\subsection{Datasets and Settings}

\subsubsection{Datasets}

To evaluate the performance of DST-SGNN\footnote{Our code at https://github.com/komorebi424/DST-SGNN.}, this study conducted experiments on seven spatio-temporal benchmark datasets: PEMS03, PEMS07, Electricity, CSI300, exchange\_rate, Solar, and METR-LA. Table \ref{tab1} provides a summary of the dataset statistics, with detailed information on the seven public benchmark datasets as follows:

\begin{itemize}
\item{PEMS03 and PEMS07: These two datasets consist of California highway traffic flow data, collected in real-time every 30 seconds by the California Performance Measurement System (PEMS) \cite{37}. The original traffic flow data is aggregated at 5-minute intervals. The datasets record the geographical information of the sensor stations. PEMS03 and PEMS07 are datasets from two specific regions.}
\item{Electricity: Electricity contains hourly time series of the electricity consumption of 321 customers from 2012 to 2014 \cite{38}.}
\item{CSI300: This is a stock dataset, compiling the closing prices of one hundred stocks that have been continuously listed in the CSI 300 from 2012 to 2024 \cite{64}.}
\item{exchange\_rate: This dataset records the daily exchange rates of eight different countries from 1990 to 2016 \cite{13}.}
\item{Solar\footnote{https://www.nrel.gov/grid/solar-power-data.html.}: The dataset selects power plant data points in Florida as the dataset, which includes 593 points. The data collection period spans from January 6, 2001, to December 31, 2006, with samples taken every hour, totaling 8,760 in length.}
\item{METR-LA: The dataset includes traffic information collected by loop detectors on highways in Los Angeles County from March 1, 2012, to June 30, 2012. It contains data from 207 sensors, with samples taken every 5 minutes. \cite{11}.}
\end{itemize}

We divided all the datasets chronologically into a training set (70\%), a validation set (10\%), and a test set (20\%).

\subsubsection{Experimental Settings}

The DST-SGNN model is implemented in Python using PyTorch 2.2.0 and is trained on a GPU (NVIDIA GeForce RTX 4090). MSE and MAE are used as evaluation metrics. We set the input window size $T$ to 96, 192, or 336, and the learning rate to 0.0001. All of the models follow the same experimental setup with forecasting window sizes \( H \in \{96, 192, 336, 720\} \) for the PEMS03, Electricity, Sola and METR-LA datasets, \( H \in \{48, 96, 192, 336\} \) for the exchange\_rate datasets, \( H \in \{36, 48, 60, 96\} \) for the CSI300 dataset, and \( H \in \{60, 96, 192, 336\} \) for the PEMS07 dataset.

\subsection{Methods for Comparison}

Our experiments comprehensively compared the forecasting performance of our model with several representative state-of-the-art (SOTA) models on seven datasets, including the advanced long-term forecasting models TimeMixer \cite{41} and TimeXer \cite{42}, the Transformer-based model iTransformer \cite{17}, the spatio-temporal forecasting model MiTSformer \cite{44}, and the Fourier transform based models FourierGNN \cite{11}, FreTS \cite{45}, Fredformer \cite{46}, and FITS \cite{47}. We used the code released in the original papers, and all models followed the same experimental settings as described in the original papers.

\subsection{Main Results}

\begin{table*}[htb]
\centering
\resizebox{\textwidth}{!}{%
\begin{tabular}{cc|cccc|cccc|cccc|cccc}
\hline
\multicolumn{2}{c|}{\multirow{2}{*}{Methods}}            & \multicolumn{4}{c|}{PEMS03}                                           & \multicolumn{4}{c|}{Electricity}                                      & \multicolumn{4}{c|}{CSI300}                                           & \multicolumn{4}{c}{METR-LA}                                           \\ \cline{3-18} 
\multicolumn{2}{c|}{}                                    & 96              & 192             & 336             & 720             & 96              & 192             & 336             & 720             & 36              & 48              & 60              & 96              & 96              & 192             & 336             & 720             \\ \hline
\multicolumn{1}{c|}{\multirow{2}{*}{DST-SGNN}}     & MSE & \textbf{0.1936} & \textbf{0.2122} & \textbf{0.2318} & {\ul 0.3040}    & \textbf{0.1552} & \textbf{0.1671} & \textbf{0.1806} & \textbf{0.2133} & \textbf{0.1435} & \textbf{0.1876} & \textbf{0.2241} & \textbf{0.2987} & \textbf{1.2756} & \textbf{1.3912} & \textbf{1.4780} & \textbf{1.7026} \\
\multicolumn{1}{c|}{}                              & MAE & \textbf{0.2635} & \textbf{0.2760} & \textbf{0.2910} & \textbf{0.3506} & \textbf{0.2448} & \textbf{0.2555} & \textbf{0.2707} & \textbf{0.2999} & \textbf{0.2182} & \textbf{0.2487} & \textbf{0.2717} & \textbf{0.3259} & \textbf{0.6243} & \textbf{0.6619} & \textbf{0.6841} & \textbf{0.7302} \\ \hline
\multicolumn{1}{c|}{\multirow{2}{*}{Rep-StdGSC}}    & MSE & {\ul 0.2177}    & 0.2475          & 0.2513          & 0.3304          & 0.2018          & {\ul 0.2026}    & 0.2168          & 0.2494          & {\ul 0.1508}    & 0.2466          & 0.2479          & 0.3850          & 1.3613          & {\ul 1.4401}    & {\ul 1.5140}    & 1.7244          \\
\multicolumn{1}{c|}{}                              & MAE & 0.2973          & 0.3301          & 0.3200          & 0.3803          & 0.2760          & {\ul 0.2818}    & 0.3011          & 0.3286          & {\ul 0.2254}    & 0.2841          & {\ul 0.2846}    & 0.3744          & 0.6420          & 0.6888          & 0.6996          & 0.7428          \\ \hline
\multicolumn{1}{c|}{\multirow{2}{*}{Rep-SpatConv}} & MSE & 0.2180          & {\ul 0.2255}    & 0.2469          & 0.3270          & 0.2021          & 0.2067          & 0.2172          & 0.2535          & 0.1687          & 0.2338          & 0.2507          & {\ul 0.3249}    & {\ul 1.2997}    & 1.5104          & 1.5474          & {\ul 1.7176}    \\
\multicolumn{1}{c|}{}                              & MAE & 0.3043          & 0.3000          & 0.3218          & 0.3798          & 0.2775          & 0.2890          & 0.3024          & 0.3342          & 0.2356          & 0.2748          & 0.2863          & {\ul 0.3421}    & {\ul 0.6377}    & {\ul 0.6843}    & {\ul 0.6974}    & {\ul 0.7347}    \\ \hline
\multicolumn{1}{c|}{\multirow{2}{*}{w/o-Stiefel}}  & MSE & 0.2199          & 0.2297          & {\ul 0.2443}    & \textbf{0.3037} & {\ul 0.2015}    & 0.2029          & {\ul 0.2145}    & {\ul 0.2484}    & 0.1599          & {\ul 0.2142}    & {\ul 0.2411}    & 0.3753          & 1.2766          & 1.4854          & 1.5107          & 1.6990          \\
\multicolumn{1}{c|}{}                              & MAE & {\ul 0.2966}    & {\ul 0.2983}    & {\ul 0.3128}    & {\ul 0.3686}    & {\ul 0.2731}    & 0.2833          & {\ul 0.2970}    & {\ul 0.3270}    & 0.2318          & {\ul 0.2628}    & 0.2859          & 0.3529          & 0.6563          & 0.6924          & 0.6976          & 0.7405          \\ \hline
\end{tabular}}
\caption{Comparative performance of DST-SGNN and variants. The best results are in bold and the second best are underlined.}
\label{tab3}
\end{table*}

Table \ref{tab2} presents the experimental results of DST-SGNN compared with other methods on the seven spatio-temporal datasets, leading to the following explanations:
\begin{itemize}
\item{The DST-SGNN model shows excellent performance on various datasets, with an average rank of 2.7857 for MSE and 1.5000 for MAE. In particular, on the CSI300 and exchange\_rate datasets, the DST-SGNN significantly outperforms other comparative methods across all forecasting window lengths. This result demonstrates that DST-SGNN has a significant advantage in handling financial data and enables more accurate forecasting of stock prices and exchange rates, which is crucial in financial decision-making.}
\item{FreTS based on Fourier transform performs well on the traffic datasets PEMS03, PEMS07, and METR-LA, but underperforms on other datasets. In contrast, iTransformer, which is based on the Transformer architecture, shows better performance on non-traffic datasets and falls short on traffic datasets. The reason is that traffic data exhibit clear periodicity and trends, which can be easily captured by Fourier transform-based methods, leading to better results on traffic datasets. Moreover, transformer-based models excel at modeling long-term dependencies and complex nonlinear patterns on non-periodic data.
 DST-SGNN is superior on all datasets by incorporating the strengths of MSGSC, which effectively extracts complex relationships in both temporal and spatial dimensions. Whether dealing with periodic traffic data or non-periodic financial and power data, DST-SGNN maintains high forecasting accuracy and efficiency.}
\item{TimeXer also performs well on all datasets, but as the forecasting length increases, the advantages of DST-SGNN become increasingly evident. This is attributed to the optimization of SGSC on the Stiefel manifold, which enables the model to more effectively handle the dynamic changes in spatio-temporal data. Additionally, DST-SGNN uses patch-based sequence decomposition to break down the original data into multiple smaller subsequences. This process further reduces the high dimensionality and complexity of the data. As a result, our model can maintain high efficiency and accuracy even with long forecasting windows. These characteristics make the advantages of DST-SGNN more prominent in long forecasting sequences and provide a more efficient and accurate solution for STTS forecasting.}
\end{itemize}

\subsection{Ablation Studies}

To further investigate the specific role of the core module SGSC in DST-SGNN, we conducted ablation studies with the following three variants:
\begin{itemize}
\item{Rep-StdGSC: Variant replacing SGSC with standard graph spectral convolution, using an orthogonal matrix $F \in \mathbb{R}^{n \times n} $ for Graph Fourier Transform to perform spectral convolutions.}
\item{Rep-SpatConv: Variant replacing SGSC with spatial graph convolution, where the convolution operation is defined as $ \sum_{i} A_i X_i W $, aggregating neighboring node features directly in the spatial domain.}
\item{w/o-Stiefel: Variant removing the Stiefel manifold constraint and replacing the original LDGOSM with a multilayer perceptron (MLP).}
\end{itemize}

The experimental results shown in Table \ref{tab3} indicate that the three variants—Rep-StdGSC, Rep-SpatConv, and w/o-Stiefel—perform worse than DST-SGNN, although they still demonstrate the effectiveness of GNNs in spatio-temporal analysis. DST-SGNN's superior performance can be attributed to its usage of the Stiefel manifold, which provides filtering capabilities by discarding low-information eigenvalues. This mechanism not only reduces noise but also enhances the model's ability to capture complex spatio-temporal patterns efficiently. Compared to Rep-StdGSC and Rep-SpatConv, DST-SGNN's other advantage lies in its computational efficiency, which will be further analyzed in the subsequent sections regarding time complexity.

\subsection{Time Consumption Analysis}

\begin{table}[htb]
\centering
\footnotesize 
\setlength{\tabcolsep}{0.43mm}{
\begin{tabular}{cc|cccc}
\hline
\multicolumn{2}{c|}{\multirow{2}{*}{Methods}}                                       & \multicolumn{4}{c}{PEMS07}                                         \\ \cline{3-6} 
\multicolumn{2}{c|}{}                                                               & 60             & 96             & 192            & 336             \\ \hline
\multicolumn{1}{c|}{\multirow{3}{*}{DST-SGNN}}      & Train                          & {\ul 236.84}   & {\ul 240.14}   & {\ul 271.23}   & {\ul 304.4}     \\
\multicolumn{1}{c|}{}                              & Inference                      & {\ul 61.7}     & \textbf{65.98} & \textbf{78.78} & \textbf{101.34} \\
\multicolumn{1}{c|}{}                              & \multicolumn{1}{l|}{Parameters} & \textbf{4176K} & \textbf{4213K} & \textbf{4311K} & \textbf{4459K}  \\ \hline
\multicolumn{1}{c|}{\multirow{3}{*}{Rep-StdGSC}}    & Train                          & 696.04         & 701.09         & 720.5          & 742.07          \\
\multicolumn{1}{c|}{}                              & Inference                      & 174.4          & 177.33         & 209.57         & 248.13          \\
\multicolumn{1}{c|}{}                              & \multicolumn{1}{l|}{Parameters} & 225163K        & 225200K        & 225298K        & 225446K         \\ \hline
\multicolumn{1}{c|}{\multirow{3}{*}{Rep-SpatConv}} & Train                          & 698.55         & 682.25         & 798.8          & 699.72          \\
\multicolumn{1}{c|}{}                              & Inference                      & 179.04         & 173.05         & 183.96         & 185.44          \\
\multicolumn{1}{c|}{}                              & \multicolumn{1}{l|}{Parameters} & 76546K         & 76583K         & 76681K         & 76829K          \\ \hline
\multicolumn{1}{c|}{\multirow{3}{*}{MiTSformer}}   & Train                          & 2933.64        & 2668.45        & 2907.72        & 2952.88         \\
\multicolumn{1}{c|}{}                              & Inference                      & 854.95         & 862.67         & 858.43         & 1379.61         \\
\multicolumn{1}{c|}{}                              & \multicolumn{1}{l|}{Parameters} & 6821K          & 6848K          & 6922K          & 7033K           \\ \hline
\multicolumn{1}{c|}{\multirow{3}{*}{Fredformer}}   & Train                          & 966.38         & 950.04         & 1088.01        & 1281.55         \\
\multicolumn{1}{c|}{}                              & Inference                      & 74.01          & 84             & 108.58         & 173.62          \\
\multicolumn{1}{c|}{}                              & \multicolumn{1}{l|}{Parameters} & 74936K         & 119803K        & 239537K        & 419379K         \\ \hline
\multicolumn{1}{c|}{\multirow{3}{*}{iTransformer}} & Train                          & \textbf{98.5}  & \textbf{98.73} & \textbf{97.55} & \textbf{106.78} \\
\multicolumn{1}{c|}{}                              & Inference                      & \textbf{56.31} & {\ul 67.3}     & {\ul 111.35}   & {\ul 182.18}    \\
\multicolumn{1}{c|}{}                              & \multicolumn{1}{l|}{Parameters} & {\ul 6393K}   & {\ul 6411K}    & {\ul 6461K}    & {\ul 6534K}     \\ \hline
\end{tabular}}
\caption{Experimental time consumption (seconds per epoch) and parameter analysis on PEMS07 dataset. The best results are in bold and the second best are underlined.}
\label{tab4}
\end{table}

Table \ref{tab4} presents the experimental time consumption and parameter count for DST-SGNN, Rep-StdGSC, Rep-SpatConv, MiTSformer, Fredformer, and iTransformer on the PEMS07 dataset. The results show that DST-SGNN achieves lower time and space complexity than most methods due to its optimized SGSC. Although DST-SGNN has a longer training time than iTransformer due to the eigenvalue decomposition in the Stiefel manifold optimization, its inference time and parameter count are lower. This is attributed to the reduced complexity of eigenvalue decomposition to $O({{d}^{3}})$ with $d\ll n$, which enables faster inference despite the additional training overhead. In addition, DST-SGNN has the fewest parameters. This highlights DST-SGNN's advantage in spatial complexity, making it more memory-efficient while also reducing inference overhead. For more experimental results, please refer to Appendix \ref{appC}.

\section{Conclusions}

The DST-SGNN model proposed in this paper effectively handles the dynamic changes and high-dimensional characteristics of spatio-temporal data by combining the advantages of graph neural networks and SGSC. DST-SGNN reduces data dimensionality and complexity through patch-based sequence decomposition and optimizes its core module, the SGSC, on the Stiefel manifold, reducing computational complexity while maintaining sensitivity to dynamic spatio-temporal relationships. Experiments show that DST-SGNN outperforms existing methods in forecasting accuracy and computational efficiency on seven public benchmark datasets, particularly in handling financial time series data and long forecasting windows. Our experiments have covered a variety of spatio-temporal data. In the future, we will consider more spatio-temporal application scenarios and further utilize and improve our model to apply it to these scenarios.

\newpage

\section*{Acknowledgments}

This research was supported by the PCL-CMCC Foundation for Science and Innovation (Grant No. 2024ZY2B0050).

\bibliographystyle{named}
\bibliography{ijcai25.bib}

\appendix

\section{Explanations and Proofs}

\subsection{Proof of Theorem 1}

\label{appA.1}

\setcounter{theorem}{0}
\setcounter{algorithm}{0}

\begin{theorem}
The matrix $F$ can be obtained by solving the eigenvalue decomposition of ${{D}^{-1/2}}A{{D}^{-1/2}}$ and selecting the eigenvectors corresponding to the d largest eigenvalues.
\label{the1A}
\end{theorem}

\begin{proof}

\begin{equation}
\begin{aligned}
  & Tr({{F}^{T}}LF)=Tr({{F}^{T}}({{I}_{n}}-{{D}^{-1/2}}A{{D}^{-1/2}})F) \\ 
 & \,\,\,\,\,\,\,\,\,\,\,\,\,\,\,\,\,\,\,\,\,\,\,\,\,=Tr({{F}^{T}}F-{{F}^{T}}{{D}^{-1/2}}A{{D}^{-1/2}}F) \\ 
 & \,\,\,\,\,\,\,\,\,\,\,\,\,\,\,\,\,\,\,\,\,\,\,\,=Tr({{I}_{n}})-Tr({{F}^{T}}{{D}^{-1/2}}A{{D}^{-1/2}}F) \\ 
\end{aligned}
\end{equation}
\noindent
where $Tr({{I}_{n}})=n$ is a constant, the objective function in Definition 1 can be equivalently transformed into the following optimization problem:
\begin{equation}
\begin{aligned}
  & Max\,\,\,Tr({{F}^{T}}{{D}^{-1/2}}A{{D}^{-1/2}}F) \\ 
 & s.t.\,\,\,\,\,\,\,\,\,\,\,\,\,\,\,\,\,\,\,\,{{F}^{T}}F=I \\ 
\end{aligned}
\end{equation}

The solution to this optimization problem can be obtained by solving the eigenvalue problem of ${{D}^{-1/2}}A{{D}^{-1/2}}$, specifically by selecting the eigenvectors corresponding to the k largest eigenvalues.

\end{proof}

\subsection{Proof of Theorem 2}

\label{appA.2}

\begin{theorem}

The Stiefel Graph Spectral Convolution can be viewed as the following specific filtered graph spectral convolution:

\begin{equation}
x{{*}_{s}}g=P{{g}_{\theta }}(\Lambda ){{P}^{T}}x
\end{equation}
\noindent
where $\Lambda =diag({{\lambda }_{1}},{{\lambda }_{2}},\cdots ,{{\lambda }_{n}})$, and ${{\lambda }_{1}},{{\lambda }_{2}},\cdots ,{{\lambda }_{n}}$ are the eigenvalues of $A$ sorted in descending order. $P=({{p}_{1}},{{p}_{2}},\cdots ,{{p}_{n}})$, where ${{p}_{i}}\in {{R}^{n\times 1}}$ is the eigenvector corresponding to ${{\lambda }_{i}}$. ${{g}_{\theta }}(\Lambda )=diag(\theta )$, and $\theta =({{\theta }_{1}},{{\theta }_{2}},\cdots ,{{\theta }_{n}})$ satisfies:
\begin{equation}
\theta_i = 
\begin{cases} 
p_i^T g & \text{if } \lambda_i \ge \lambda_d \\
0 & \text{if } \lambda_i < \lambda_d 
\end{cases}
\end{equation}
\label{the2A}
\end{theorem}

\begin{proof}

According to the conditions, ${{\lambda }_{1}},{{\lambda }_{2}},\cdots ,{{\lambda }_{n}}$ are the eigenvalues of $A$ sorted in descending order. Let $f=({{f}_{1}},\cdots ,{{f}_{d}})$. By Theorem \ref{the1A}, ${{f}_{i}}(i=1,\cdots ,d)$ is the eigenvector corresponding to ${{\lambda }_{i}}$, then:

\begin{equation}
\begin{aligned}
  & x{{*}_{s}}g=F({{F}^{T}}x\odot {{F}^{T}}g) \\ 
 & \text{         }=({{f}_{1}},{{f}_{2}},\cdots ,{{f}_{d}})\left( \left( \begin{matrix}
   f_{1}^{T}x  \\
   f_{2}^{T}x  \\
   \vdots   \\
   f_{d}^{T}x  \\
\end{matrix} \right)\odot \left( \begin{matrix}
   f_{1}^{T}g  \\
   f_{2}^{T}g  \\
   \vdots   \\
   f_{d}^{T}g  \\
\end{matrix} \right) \right) \\ 
 & \text{         }=({{f}_{1}},{{f}_{2}},\cdots ,{{f}_{d}})\left( \begin{matrix}
   (f_{1}^{T}x)\cdot (f_{1}^{T}g)  \\
   (f_{2}^{T}x)\cdot (f_{2}^{T}g)  \\
   \vdots   \\
   (f_{d}^{T}x)\cdot (f_{d}^{T}g)  \\
\end{matrix} \right) \\ 
 & \text{         }=\sum\limits_{i=1}^{d}{{{f}_{i}}((f_{2}^{T}x)\cdot (f_{2}^{T}g))} \\ 
\end{aligned}
\end{equation}

Given ${{g}_{\theta }}(\Lambda )=diag(\theta )$, where $\theta =({{\theta }_{1}},{{\theta }_{2}},\cdots ,{{\theta }_{n}})$ and ${{\theta }_{i}}=\left\{ \begin{matrix}
   p_{i}^{T}G,{{\lambda }_{i}}\ge {{\lambda }_{d}}  \\
   0,{{\lambda }_{i}}<{{\lambda }_{d}}  \\
\end{matrix} \right.$. Let $P=({{p}_{1}},{{p}_{2}},\cdots ,{{p}_{n}})$, where ${{p}_{i}}(i=1,\cdots ,n)$ is the eigenvector corresponding to ${{\lambda }_{i}}$. It is clear that ${{f}_{i}}={{p}_{i}},i=1,\cdots ,d$. Let ${{g}_{\theta }}(\Lambda )=diag(\theta )$, then we have:
\begin{equation}
\begin{aligned}
  & \text{\,\,\,\,\,\,\,\,}P{{g}_{\theta }}(\Lambda ){{P}^{T}}x \\ 
 & =({{p}_{1}},{{p}_{2}},\cdots ,{{p}_{n}})\left( \begin{matrix}
   p_{1}^{T}g & \cdots  & 0 & 0 & \cdots  & 0  \\
   \vdots  & \ddots  & \vdots  & \vdots  & \ddots  & \vdots   \\
   0 & \cdots  & p_{d}^{T}g & 0 & \cdots  & 0  \\
   0 & \cdots  & 0 & 0 & \cdots  & 0  \\
   \vdots  & \ddots  & \vdots  & \vdots  & \ddots  & \vdots   \\
   0 & \cdots  & 0 & 0 & \cdots  & 0  \\
\end{matrix} \right)\left( \begin{matrix}
   p_{1}^{T}x  \\
   p_{2}^{T}x  \\
   \vdots   \\
   p_{d}^{T}x  \\
   \vdots   \\
   p_{n}^{T}x  \\
\end{matrix} \right) \\ 
 & =({{p}_{1}},{{p}_{2}},\cdots ,{{p}_{n}})\left( \begin{matrix}
   (p_{1}^{T}x)\cdot (p_{1}^{T}g)  \\
   (p_{2}^{T}x)\cdot (p_{2}^{T}g)  \\
   \vdots   \\
   (p_{d}^{T}x)\cdot (p_{d}^{T}g)  \\
   0  \\
   \vdots   \\
   0  \\
\end{matrix} \right) \\ 
 & =\sum\limits_{i=1}^{d}{{{p}_{i}}}((p_{i}^{T}x)\cdot (p_{i}^{T}g)) \\ 
\end{aligned}
\end{equation}

Given ${{f}_{i}}={{p}_{i}},i=1,\cdots ,d$, we can derive that $x{{*}_{S}}g=P{{g}_{\theta }}(\Lambda ){{P}^{T}}x$.

\end{proof}

\subsection{Proof of Algorithm 1}

\label{appA.3}

\begin{algorithm}[h]
    \caption{Stiefel manifold optimization algorithm}
    \label{alg1A}
    \textbf{Input}: $X\in {{R}^{n\times d}}$, $E\in {{R}^{n\times d}}$\\
    \textbf{Output}: $W\in {{R}^{d\times d}}$
    \begin{algorithmic}[1] 
        \STATE Perform eigenvalue decomposition on $B={{X}^{T}}X$ to obtain eigenvector matrix $D$ and eigenvalue diagonal matrix $\Lambda $;
        \STATE Let $M=D{{\Lambda }^{-1/2}}{{D}^{T}}$;
        \STATE Compute $C={{E}^{T}}X$;
        \STATE Compute $H=B+{{C}^{T}}C$;
        \STATE Perform eigenvalue decomposition on ${{M}^{T}}HM$ to obtain eigenvector matrix $U$;
        \STATE \textbf{return} $W=MU$;
    \end{algorithmic}
\end{algorithm}

\begin{proof}
Let $D$ and $\Lambda $ be the eigenvector matrix and the diagonal matrix of eigenvalues of the matrix ${{X}^{T}}X$, respectively, such that ${{X}^{T}}X=D\Lambda {{D}^{T}}$. Taking $M=D{{\Lambda }^{-1/2}}{{D}^{T}}$, and computing $W=MU$, we have:

\begin{equation}
\begin{aligned}
U={{M}^{-1}}W=D{{\Lambda }^{1/2}}{{D}^{T}}W
\end{aligned}
\end{equation}

and:

\begin{equation}
\begin{aligned}
  & {{U}^{T}}U={{W}^{T}}D{{\Lambda }^{1/2}}{{D}^{T}}D{{\Lambda }^{1/2}}{{D}^{T}}W\,={{W}^{T}}D\Lambda {{D}^{T}}W \\ 
 & \,\,\,\,\,\,\,\,\,\,\,\,\,={{W}^{T}}{{X}^{T}}XW=I \\ 
\end{aligned}
\end{equation}

Substituting $W=MU$ and ${{U}^{T}}U=I$ into the objective function (corresponding to equation (11) in the main manuscript):
\begin{equation}
\begin{aligned}
& \max \quad \text{Tr}(W^T (X^T X + X^T E E^T X) W) \\
& \text{s.t.} \quad \quad \quad \quad \quad W^T X^T X W = I \\
\end{aligned}
\end{equation}

The optimization problem for $W$ is transformed into the following optimization problem for $U$:

\begin{equation}
\begin{aligned}
  & \max Tr({{U}^{T}}{{M}^{T}}({{X}^{T}}X+{{X}^{T}}E{{E}^{T}}X)MU) \\ 
 & \,\,\,\,\,\,\,\,\,\,\,\,\,\,\,\,\,\,\,\,\,\,\,\,\,\,\,\,\,\,\,\,\,\,\,\, s.t.  \,\,\,{{U}^{T}}U=I \\ 
\end{aligned}
\end{equation}

Let $H={{X}^{T}}X+{{X}^{T}}E{{E}^{T}}X$, then the above optimization problem is equivalent to finding the eigenvector matrix $U$ of ${{M}^{T}}HM$.

\end{proof}

\subsection{Proof of Theorem 3}

\label{appA.4}

To prove Theorem \ref{the3}, we first introduce two lemmas concerning the properties of the Stiefel Graph Fourier Transform.

\begin{lemma}
The Stiefel Graph Fourier Transform satisfies:

\begin{equation}
\begin{aligned}
S\left( X{{*}_{s}}{{G}_{1}}{{*}_{s}}{{G}_{2}}{{*}_{s}}\cdots {{*}_{s}}{{G}_{i}} \right)=S(X)\odot \prod\limits_{j=1}^{i}{S(}{{G}_{j}})
\end{aligned}
\end{equation}

\end{lemma}

\begin{proof}
Given $X{{*}_{S}}G=F({{F}^{T}}X\odot {{F}^{T}}G)$, it follows that $S(X{{*}_{S}}G)=S(X)\odot S(G)$, therefore:

\begin{equation}
\begin{aligned}
  & \text{\,\,\,\,\,\,\,   }S\left( X{{*}_{s}}{{G}_{1}}{{*}_{s}}{{G}_{2}}{{*}_{s}}\cdots {{*}_{s}}{{G}_{i}} \right) \\ 
 & =S\left( X{{*}_{s}}{{G}_{1}}{{*}_{s}}{{G}_{2}}{{*}_{s}}\cdots {{*}_{s}}{{G}_{i-1}} \right)\odot S({{G}_{i}}) \\ 
 & =\cdots =S(X)\odot \prod\limits_{j=1}^{i}{S(}{{G}_{j}}) \\ 
\end{aligned}
\end{equation}

\end{proof}

\begin{lemma}
The Stiefel Graph Fourier Transform satisfies:

\begin{equation}
\begin{aligned}
S\left( \sum\limits_{i=1}^{m}{{{Y}_{i}}} \right)=\sum\limits_{i=1}^{m}{S({{Y}_{i}})}
\end{aligned}
\end{equation}

\end{lemma}

\begin{proof}

\begin{equation}
\begin{aligned}
 S\left( \sum\limits_{i=1}^{m}{{{Y}_{i}}} \right)={{F}^{T}}\left( \sum\limits_{i=1}^{m}{{{Y}_{i}}} \right)=\sum\limits_{i=1}^{m}{{{F}^{T}}{{Y}_{i}}}=\sum\limits_{i=1}^{m}{S({{Y}_{i}})}
\end{aligned}
\end{equation}

\end{proof}

\begin{theorem}
The MSGSC has the following equivalent computational form:
\begin{equation}
MSGSC(X,G)={{S}^{-1}}(\sum\limits_{i=1}^{m}{S(X)}\odot \prod\limits_{j=1}^{i}{S({{G}_{j}})})
\end{equation}
\label{the3A}
\end{theorem}

\begin{proof}

\begin{equation}
\begin{aligned}  
& \text{\,\,\,}S\left( \sum\limits_{i=1}^{m}{X{{*}_{s}}{{G}_{1}}{{*}_{s}}{{G}_{2}}{{*}_{s}}\cdots {{*}_{s}}{{G}_{i}}} \right) \\ 
 & =\sum\limits_{i=1}^{m}{S(X{{*}_{s}}{{G}_{1}}{{*}_{s}}{{G}_{2}}{{*}_{s}}\cdots {{*}_{s}}{{G}_{i}})} \\ 
 & =\sum\limits_{i=1}^{m}{S(X)}\odot \prod\limits_{j=1}^{i}{S({{G}_{j}})} \\ 
\end{aligned}
\end{equation}

Therefore, 

\begin{equation}
\begin{aligned}
  & \text{\,\,\,\,\,\,\,} MSGSC(X,G) \\ 
 & =\sum\limits_{i=1}^{m}{X{{*}_{s}}{{G}_{1}}{{*}_{s}}{{G}_{2}}{{*}_{s}}\cdots {{*}_{s}}{{G}_{i}}} \\ 
 & ={{S}^{-1}}\left( \sum\limits_{i=1}^{m}{S(X)}\odot \prod\limits_{j=1}^{i}{S({{G}_{j}})} \right) \\ 
\end{aligned}
\end{equation}

\end{proof}

\section{Parameter Analysis}

\label{appB.1}

\subsection{The impact of the dimension of matrix \texorpdfstring{$F$}{F}}

\begin{table}[ht]
\centering
\footnotesize 
\setlength{\tabcolsep}{0.7mm}{
\begin{tabular}{cl|cc|cc}
\hline
\multicolumn{2}{c|}{\multirow{2}{*}{Methods}} & \multicolumn{2}{c|}{CSI300}       & \multicolumn{2}{c}{Solar}         \\ \cline{3-6} 
\multicolumn{2}{c|}{}                         & MAE             & MSE             & MAE             & MSE             \\ \hline
\multicolumn{2}{c|}{d\_model=64}              & 0.3438          & 0.3419          & 0.2027          & {\ul 0.2132}    \\ \hline
\multicolumn{2}{c|}{d\_model=128}             & \textbf{0.2987} & \textbf{0.3259} & \textbf{0.1802} & \textbf{0.2007} \\ \hline
\multicolumn{2}{c|}{d\_model=256}             & {\ul 0.3022}    & {\ul 0.3281}    & 0.1919          & 0.2150          \\ \hline
\multicolumn{2}{c|}{d\_model=512}             & 0.3057          & 0.3281          & {\ul 0.1916}    & 0.2077          \\ \hline
\end{tabular}
}
\caption{The impact of the dimension of matrix $F$ on the results under a prediction length of 96, evaluated on the CSI300 and Solar datasets. The best results are in bold and the second best are underlined.}
\label{tab5}
\end{table}

As shown in Table \ref{tab5}, the impact of the dimension of matrix $F$ on the results under a prediction length of 96 was evaluated on the CSI300 and Solar datasets. Among the tested dimensions (64, 128, 256, and 512), a dimension of 128 was found to be optimal. This suggests that a moderate dimensionality of 128 provides a good balance between model capacity and computational efficiency, allowing the model to capture essential features without excessive complexity or computational overhead.

\subsection{The impact of the number of \texorpdfstring{$HP$}{HP} stacking layers}

\label{appB.2}

\begin{table}[ht]
\centering
\footnotesize 
\setlength{\tabcolsep}{0.7mm}{
\begin{tabular}{cl|cc|cc}
\hline
\multicolumn{2}{c|}{\multirow{2}{*}{Methods}} & \multicolumn{2}{c|}{CSI300}             & \multicolumn{2}{c}{Solar}         \\ \cline{3-6} 
\multicolumn{2}{c|}{}                         & MAE                   & MSE             & MAE             & MSE             \\ \hline
\multicolumn{2}{c|}{1}                        & 0.3501                & 0.3457          & 0.2036          & 0.2307          \\ \hline
\multicolumn{2}{c|}{2}                        & {\ul 0.2987}          & \textbf{0.3259} & \textbf{0.1802} & \textbf{0.2007} \\ \hline
\multicolumn{2}{c|}{3}                        & {\ul \textbf{0.2914}} & {\ul 0.3308}    & {\ul 0.1918}    & {\ul 0.2031}    \\ \hline
\multicolumn{2}{c|}{4}                        & 0.3330                & 0.3407          & 0.1936          & 0.2107          \\ \hline
\end{tabular}
}
\caption{The impact of the number of $HP$ stacking layers on the results under a prediction length of 96, evaluated on the CSI300 and Solar datasets. The best results are in bold and the second best are underlined.}
\label{tab6}
\end{table}

As shown in Table \ref{tab6}, under a prediction length of 96, the impact of the number of $HP$ stacking layers on the results was evaluated on the CSI300 and Solar datasets. Among the tested configurations (1, 2, 3, and 4 layers), two stacking layers were found to be optimal. This indicates that a moderate number of stacking layers can effectively capture the complex spatio-temporal dependencies in the data without introducing unnecessary model complexity or computational burden. Two layers provide sufficient depth to model the underlying patterns while maintaining efficiency.

\section{Time Consumption Analysis}

\label{appC}

\begin{table}[ht]
\centering
\footnotesize 
\setlength{\tabcolsep}{0.7mm}{
\begin{tabular}{cc|cccc}
\hline
\multicolumn{2}{c|}{\multirow{2}{*}{Methods}}                  & \multicolumn{4}{c}{Solar}                                         \\ \cline{3-6} 
\multicolumn{2}{c|}{}                                          & 96             & 192            & 336            & 720            \\ \hline
\multicolumn{1}{c|}{\multirow{3}{*}{DST-SGNN}}     & Train     & \textbf{58.09} & \textbf{62.05} & \textbf{62.86} & \textbf{66.53} \\
\multicolumn{1}{c|}{}                              & Inference & \textbf{15.18} & \textbf{17.59} & \textbf{17.65} & \textbf{21.1}  \\
\multicolumn{1}{c|}{}                              & Parameter & \textbf{3318K} & \textbf{3417K} & \textbf{3564K} & \textbf{3958K} \\ \hline
\multicolumn{1}{c|}{\multirow{3}{*}{MiTSformer}}   & Train     & 444.37         & 439.62         & 436.22         & 415.8          \\
\multicolumn{1}{c|}{}                              & Inference & 130.44         & 123.9          & 125.62         & 90.06s         \\
\multicolumn{1}{c|}{}                              & Parameter & 6848K          & 6922K          & 7033K          & 7329K          \\ \hline
\multicolumn{1}{c|}{\multirow{3}{*}{Fredformer}}   & Train     & 229.37         & 216.84         & 311.81         & 890.46         \\
\multicolumn{1}{c|}{}                              & Inference & {\ul 19.05}    & {\ul 21.8}     & {\ul 31.32}    & {\ul 58.89}    \\
\multicolumn{1}{c|}{}                              & Parameter & 160741K        & 321375K        & 562567K        & 605942K        \\ \hline
\multicolumn{1}{c|}{\multirow{3}{*}{iTransformer}} & Train     & {\ul 98.5}     & {\ul 98.73}    & {\ul 97.55}    & {\ul 106.78}   \\
\multicolumn{1}{c|}{}                              & Inference & 56.31          & 67.3           & 111.35         & 182.18         \\
\multicolumn{1}{c|}{}                              & Parameter & {\ul 6411K}    & {\ul 6461K}    & {\ul 6534K}    & {\ul 6534K}    \\ \hline
\end{tabular}}
\caption{Experimental time consumption (seconds per epoch) and parameter analysis on Solar dataset. The best results are in bold and the second best are underlined.}
\label{tab7}
\end{table}

\begin{table}[ht]
\centering
\footnotesize 
\setlength{\tabcolsep}{0.7mm}{
\begin{tabular}{cc|cccc}
\hline
\multicolumn{2}{c|}{\multirow{2}{*}{Methods}}                  & \multicolumn{4}{c}{PEMS03}                                        \\ \cline{3-6} 
\multicolumn{2}{c|}{}                                          & 96             & 192            & 336            & 720            \\ \hline
\multicolumn{1}{c|}{\multirow{3}{*}{DST-SGNN}}     & Train     & {\ul 171.94}   & {\ul 177.98}   & {\ul 186.97}   & {\ul 200.65s}  \\
\multicolumn{1}{c|}{}                              & Inference & 45.14          & 54.8           & 67.81          & 93.17          \\
\multicolumn{1}{c|}{}                              & Parameter & \textbf{2594K} & \textbf{2692K} & \textbf{2840K} & \textbf{3233K} \\ \hline
\multicolumn{1}{c|}{\multirow{3}{*}{MiTSformer}}   & Train     & 1105.19        & 1118.38        & 1101.96        & 1098.73        \\
\multicolumn{1}{c|}{}                              & Inference & 318.67         & 334.34         & 319.64         & 310.16         \\
\multicolumn{1}{c|}{}                              & Parameter & 6848K          & 6922K          & 7033K          & 7329K          \\ \hline
\multicolumn{1}{c|}{\multirow{3}{*}{Fredformer}}   & Train     & 598.41         & 558.08         & 557.65         & 602.74         \\
\multicolumn{1}{c|}{}                              & Inference & {\ul 39.54}    & {\ul 47.69}    & {\ul 52.95}    & {\ul 69.71}    \\
\multicolumn{1}{c|}{}                              & Parameter & {\ul 48688K}   & {\ul 62153K}   & 108995K        & 367347K        \\ \hline
\multicolumn{1}{c|}{\multirow{3}{*}{iTransformer}} & Train     & \textbf{22.93} & \textbf{23.45} & \textbf{24.67} & \textbf{28.73} \\
\multicolumn{1}{c|}{}                              & Inference & \textbf{15.19} & \textbf{23.57} & \textbf{35.16} & \textbf{65.35} \\
\multicolumn{1}{c|}{}                              & Parameter & 6411K          & 6461K          & {\ul 6534K}    & {\ul 6731K}    \\ \hline
\end{tabular}}
\caption{Experimental time consumption (seconds per epoch) and parameter analysis on PEMS03 dataset. The best results are in bold and the second best are underlined.}
\label{tab8}
\end{table}

\begin{table}[ht]
\centering
\footnotesize 
\setlength{\tabcolsep}{0.7mm}{
\begin{tabular}{cc|cccc}
\hline
\multicolumn{2}{c|}{\multirow{2}{*}{Methods}}                  & \multicolumn{4}{c}{exchange\_rate}                                \\ \cline{3-6} 
\multicolumn{2}{c|}{}                                          & 48             & 96             & 192            & 336            \\ \hline
\multicolumn{1}{c|}{\multirow{3}{*}{DST-SGNN}}     & Train     & 37.75          & 37.08          & 36.36          & 34.23          \\
\multicolumn{1}{c|}{}                              & Inference & 8.65           & 8.37           & 7.78           & 6.93           \\
\multicolumn{1}{c|}{}                              & Parameter & {\ul 1465K}    & {\ul 1514K}    & \textbf{1613K} & \textbf{1760K} \\ \hline
\multicolumn{1}{c|}{\multirow{3}{*}{MiTSformer}}   & Train     & 23.28          & 23.13          & 22.91          & 22.16          \\
\multicolumn{1}{c|}{}                              & Inference & 10.5           & 10.56          & 9.64           & 8.8            \\
\multicolumn{1}{c|}{}                              & Parameter & 6811K          & 6848K          & 6922K          & 7033K          \\ \hline
\multicolumn{1}{c|}{\multirow{3}{*}{Fredformer}}   & Train     & {\ul 11.79}    & {\ul 11.77}    & {\ul 11.71}    & {\ul 10.43}    \\
\multicolumn{1}{c|}{}                              & Inference & \textbf{0.87}  & \textbf{0.73}  & \textbf{0.72}  & \textbf{0.8}   \\
\multicolumn{1}{c|}{}                              & Parameter & \textbf{680K}  & \textbf{1278K} & {\ul 2572K}    & {\ul 4754K}    \\ \hline
\multicolumn{1}{c|}{\multirow{3}{*}{iTransformer}} & Train     & \textbf{10.81} & \textbf{10.91} & \textbf{10.71} & \textbf{10.41} \\
\multicolumn{1}{c|}{}                              & Inference & {\ul 1.59}     & {\ul 1.39}     & {\ul 1.68}     & {\ul 1.23}     \\
\multicolumn{1}{c|}{}                              & Parameter & 6387K          & 6411K          & 6461K          & 6534K          \\ \hline
\end{tabular}}
\caption{Experimental time consumption (seconds per epoch) and parameter analysis on exchange\_rate dataset. The best results are in bold and the second best are underlined.}
\label{tab9}
\end{table}

As Table \ref{tab7}, Table \ref{tab8}, and Table \ref{tab9} show, our method demonstrates unique advantages in terms of the number of parameters and time efficiency. In terms of the number of parameters, our method shows an advantage across all the tested datasets. Regarding time efficiency, its performance is closely related to the variable scale of the dataset. Especially on datasets with a rich number of variables, the advantage in operational efficiency is even more prominent.

Taking the Solar dataset as an example, our method is relatively leading in terms of time consumption and has the fastest running speed. During the training process on the PEMS03 dataset, compared with the two spatio-temporal models, Fredformer and MiTSformer, the training time of our method is significantly shortened, which also proves its high efficiency. However, on the exchange\_rate dataset, our method performs relatively average. This is mainly because this dataset has only 8 variables and is small in scale, so it is unable to fully utilize the advantages of our method in handling large-scale variables and complex graph structures.

It can be seen that when facing datasets with a large number of variables, our method can fully unleash its performance potential, achieve more efficient operation and processing, and has significant value in practical applications.

\end{document}